\documentclass[letterpaper]{article} 
\usepackage{aaai2026}  
\usepackage{times}  
\usepackage{helvet}  
\usepackage{courier}  
\usepackage[hyphens]{url}  
\usepackage{graphicx} 
\usepackage{natbib}  
\usepackage{caption} 
\usepackage[ruled,vlined,linesnumbered]{algorithm2e}
\usepackage{comment}

\usepackage{newfloat}
\usepackage{listings}
\DeclareCaptionStyle{ruled}{labelfont=normalfont,labelsep=colon,strut=off} 
\usepackage{amsmath,amsthm,amssymb}
\usepackage{booktabs,multirow,array}
\usepackage{pgfplots}
\pgfplotsset{compat=1.17}
\usepackage{xspace,accents}
\usepackage{enumitem}
\usepackage{subcaption}
\usepackage{caption}
\newtheorem{theorem}{Theorem}
\newtheorem{lemma}{Lemma}
\newtheorem{definition}{Definition}
\newtheorem{corollary}{Corollary}

\usepackage{listings}
\newcommand{\openf}{\textsc{Open\textsubscript{F}}\xspace}      
\newcommand{\openb}{\textsc{Open\textsubscript{B}}\xspace}      

\newcommand{\gminf}{\mbox{$gMin_F$}\xspace}     
\newcommand{\gminb}{\mbox{$gMin_B$}\xspace}   

\newcommand{\gmaxf}
{\mbox{$gMax_F$}\xspace} 
\newcommand{\gmaxb}{\mbox{$gMax_B$}\xspace} 

\newcommand{\mmm}
{\mbox{XMM}\xspace} 

\newcommand{\start}{\ensuremath{start}\xspace}
\newcommand{\goal}{\ensuremath{goal}\xspace}

\newcommand{\BiHS}{\mbox{BiHS}\xspace} 
\newcommand{\bihs}{\mbox{BiHS}\xspace} 
 
\newcommand{\gD}{\mbox{$g_D$}}
\newcommand{\gF}{\mbox{$g_F$}}
\newcommand{\gB}{\mbox{$g_B$}}

\newcommand{\hD}{\mbox{$h_D$}}
\newcommand{\hF}{\mbox{$h_F$}}
\newcommand{\hB}{\mbox{$h_B$}}

\newcommand{\fD}{\mbox{$f_D$}}

\newcommand{\Cstar}{C\ensuremath{^*}}

\definecolor{darkgreen}{rgb}{0.0, 0.5, 0.0}
\definecolor{forestgreen}{HTML}{1A5235}

\iftrue
    \usepackage{xcolor}

    \newcommand{\af}[1]{\textcolor{blue}{[Ariel: #1]}}
    
    \newcommand{\CRC}[2]{#2}
    
\else
    \newcommand{\roni}[1]{}
    \newcommand{\jonathan}[1]{}
    \newcommand{\af}[1]{}
    \newcommand{\noy}[1]{}
\fi

\title{Bidirectional Search for Longest Paths: Case for Front-to-Front Heuristics}
\begin{document}

\author{Tzur Shubi,
Ariel Felner, Solomon Eyal Shimony, Shahaf S. Shperberg}

\affiliations{Stein Faculty of Computer and Information Science, Ben-Gurion University of the Negev, Israel}

\maketitle

\begin{abstract}

Bidirectional heuristic search can potentially reduce search effort for problems amenable to backward search. Therein, it is well-known that front-to-front heuristics can reduce the number of node expansions, but their overhead is so high that overall runtime almost always increases. 
\CRC{R4C1}{We propose BiXDFBnB, a bidirectional depth-first branch-and-bound algorithm that adapts the Single-Frontier Bidirectional Search (SFBDS) framework---originally developed for shortest-path (MIN) problems---to the Generalized Longest Simple Path (GLSP) setting. Because SFBDS inherently operates on paired states, front-to-front (F2F) heuristic evaluation arises naturally and avoids the overhead typically associated with bidirectional frontier management. We show that this adaptation can be successfully applied to maximization (MAX) problems while efficiently handling overlapping constraints. BiXDFBnB is applied to several types of longest-path problems: Longest Simple Path (LSP), Snakes, and Coil-in-the-Box (CIB). Empirical evaluation shows that the new algorithm frequently reduces the number of node expansions and, in some cases, also improves overall runtime. }


  \end{abstract}

\section{Introduction}

Bidirectional heuristic search can potentially reduce search effort for combinatorial search problems amenable to backward search. Therein, it is well-known that front-to-front heuristics can reduce the number of expansions in search, but their overhead is so high that overall runtime almost always increases. This paper examines whether front-to-front heuristics can be beneficially applied in longest-path search.

The aim in the {\em shortest path problem} is to minimize a path cost between two given vertices. Such problems are called {\em minimization} (MIN) problems. A prominent problem in graph theory is to find the {\em longest simple path} (LSP) between two given vertices in an undirected graph. A {\em simple path} is one where no vertex is visited more than once. In LSP, the aim is to find a maximum-cost path. Such problems are called {\em maximization} (MAX) problems.
LSP is known to be NP-hard, and even hard to approximate within a constant factor~\citep{karger1997approximating}. The motivation for solving LSP comes from a variety of domains such as information retrieval on peer-to-peer networks~\citep{Wong:2005:IRP:1062745.1062799}, estimating the worst packet delay of Switched Ethernet network~\citep{DBLP:conf/sies/SchmidtS10}, multi-robot patrolling~\citep{Portugal:2010:MAM:1774088.1774360}, and VLSI design where the longest path should be found between two components on a printed circuit board \citep{chen2016vlsi}. Another important variant of LSP is {\em Snake-in-the-box} (SIB)~\citep{Kautz:1958:UDE}. In a {\em snake}, a path may not use neighbors of vertices that are already in the path. A {\em box} is an $n$-dimensional cube. SIB can help identify efficient error-correcting codes.

Several prior works treated LSP as a heuristic search problem:~\citet{DBLP:conf/socs/SternKPFR14} modified heuristic search algorithms designed for MIN problems to solve MAX problems. Follow-up work~\cite{DBLP:conf/aips/CohenSF20} proposed refinements, including
grid-based heuristics for LSP, and introduced several admissible heuristics for the SIB problem \citep{DBLP:conf/socs/PalomboSPFKR15other}.

LSP and SIB were recast within a {\em generalized longest simple path} (GLSP) framework that includes these problems as special cases \citep{ShimonyEtAl2022SOCS}. Most prior work on GLSP consists of unidirectional heuristic search algorithms, such as A* and branch-and-bound, that were modified to work with MAX problems. \CRC{R4C1}{Recent work \cite{ShubiSFS25} proposed XMM, a bidirectional search algorithm for GLSP that uses front-to-end heuristics in an attempt to approximate a meet-in-the-middle approach. Unlike in MIN problems, where lower-bound heuristics enable algorithms to prevent the expansion of nodes beyond the optimal solution’s midpoint (\textit{restrained search}), bidirectional search for GLSP (MAX) relies on overestimating upper bounds. This makes it hard to identify exactly when the search frontiers have crossed the midpoint.}

In bidirectional GLSP search, when forward and backward frontiers meet, one must verify that the concatenated paths are valid (exclusion constraints like vertex overlaps not violated). Maintaining separate frontiers and validating all crossing path pairs incurs large overhead. To \CRC{}{address} this, we introduce BiXDFBnB, a depth-first bidirectional branch-and-bound algorithm that restructures the Single-Frontier Bidirectional Search (SFBDS) framework~\cite{SFBDS}---originally designed for MIN problems---to conquer the challenges of the GLSP (MAX) setting. 

\CRC{R4C1}{Instead of growing two independent frontiers, BiXDFBnB maintains just one pair of (forward, backward) states  per search node. Given this synchronized paired-state structure, utilizing a pairwise front-to-front (F2F) heuristic is the natural design choice, entirely eliminating the frontier-matching bottleneck. However, porting this architecture from MIN to MAX is non-trivial due to the reliance on overestimating upper bounds and the inability to prune beyond a guaranteed midpoint. A core algorithmic novelty of this paper is in  effectively utilizing F2F bounds to prune the MAX search space. In addition, we integrate simultaneous Cartesian node expansion to guarantee valid path-meetings without parity issues, We then provide theoretical guarantees on the structural properties of the resulting algorithm followed by empirical evaluation validating its advantages.}



An additional contribution is that for unit edge costs, BiXDFBnB "meets in the middle", a property not achievable in \mmm.
An empirical evaluation strongly supports the case for F2F heuristics in BiXDFBnB. 
In many cases, BiXDFBnB also outperforms XMM and unidirectional search.

\section{Background}


\subsection{Generalized Longest Simple Path Problems}

The Generalized Longest Simple Path (GLSP) framework \citep{ShimonyEtAl2022SOCS}
includes both LSP and Snakes. 
Let $G=(V,E,w)$ be a connected undirected weighted graph.
A path from $v_0$ to $v_m$ is an alternating sequence $P=(v_0,(v_0,v_1),v_1,(v_1,v_2),\ldots,v_m)$ 
of vertices and edges
in $G$ such that consecutive elements in the sequence are
adjacent, i.e., each edge is incident on the 
preceding and following vertices in $P$.

A (global) {\em exclusion constraint} $L$ is a function mapping each
$x\in (V \cup E)$
to a subset of $V \cup E$.
The constraint means that {\bf after}  exiting $x$, a path may not visit any member of $L(x)$.
For example, the constraint: $\forall x\in V,~ L(x)=\{x\}$. i.e. ``no vertex may be visited more than once" defines the longest simple path (LSP) problem (Elements $x$ for which $L(x)$ is not specified - edges here - are assumed to have no constraint).
To define {\em snake problems} (not necessarily "in a box"), we have: 
$\forall x\in  V,~ L(x)=\{x\}
\cup N(x)$,
where $N(x)$ are the immediate neighbors of $x$.

\begin{definition}[Generalized LSP Problem] Given a graph
$G=(V,E,w)$, and a constraint $L$, find a path of maximum weight $w$ in $G$ (optionally starting at start vertex $s$, optionally ending at target vertex $t$) that does not violate $L$.
\end{definition}

In this paper, we only
examine the case where $G$ is unweighted (\( w(e) = 1 \;\; \forall e \in E \)), 
and where the path must start at $s$ and end at $t$.

\subsection{Best-First Exhaustive
Search in GLSP}

A {\em state} in the state-space is a path $\pi$ beginning at start vertex $s$. We assume, unless stated otherwise, that forward search begins with start state $\pi=(s)$,  and progresses
by adding one edge and vertex at a time to a path $\pi$ that has $s'$ as a last vertex.  A {\em goal state} is a path ending at target vertex $t$.
Following \citet{DBLP:conf/socs/SternKPFR14}, for a search node $N$, we denote its state, i.e., the (partial) path it represents, by
$N.\pi$, its most recently added vertex $s'$ by $N.head$, and the rest of the vertices in $N.\pi$ by $N.tail$. 
For standard unidirectional search, $g(N)$ is naturally the length of $N.\pi$.
The successors of $N.\pi$ are denoted by $\Gamma (N.\pi)$.
The {\em remaining graph}
$G_r(G,N.\pi)$ is defined to be $G$ after
removing (1) all $N.\pi$ vertices other than $N.head$ from $G$, and (2) all other elements no longer
allowed under constraint $L$.
Clearly, a node $N$ is not a dead end only if $N.head$ and $t$ are in the same connected component of $G_r(G,N.\pi)$.

In shortest-path problems (and in MIN problems in general), search algorithms such as A*
prioritize nodes with lower $f(N)=g(N)+h(N)$ values, where $g$ is the path cost from the start state to $N$, and $h$ is a cost estimation of the {\bf shortest} path from $N$ to the goal.
\citet{DBLP:conf/socs/SternKPFR14} present modifications that need to be done in order to modify A* (commonly used for MIN problems) to MAX problems. It is common practice to use all these modifications when implementing A* for MAX problems. A prominent modification for MAX is that A* now prioritizes nodes with the largest $f(N)$ values, as done specifically for LSP, where $h$ in GLSP estimates the length of the {\bf longest} path to the goal.

In GLSP, duplicate states do not occur because the state is a path, rather than just a head vertex.
Still, one can remove symmetric states: in LSP states with the same head vertex, and  the same set of visited vertices (but in a different order)~\cite{DBLP:conf/aips/CohenSF20}.

A function $h$ is said to be {\em admissible} for GLSP (i.e., MAX problems) iff for every node $N$ in the search space, $h(N)$ is {\em greater than or equal to} the  weight of the longest constrained
path (longest simple path in LSP)
from $N.head$ to $t$ in $G_r(G,N.\pi)$.
GLSP heuristics 
use graph connectivity \citep{ShimonyEtAl2022SOCS,DBLP:conf/socs/PalomboSPFKR15other}.
Trivially, an admissible heuristic (upper bound) is the entire remaining graph $G_r$ size. Obviously, it is also admissible to delete vertices no longer reachable
in $G_r(G,N.\pi)$ from the path head $N.head$ or from the target $t$. 
Finally, one can delete vertices $v \in G_r(G,N.\pi)$ for which there is no simple path from $N.head$ to $t$ containing $v$. 
\CRC{R1C3, R4C3}{This heuristic is denoted by $h_{BCC}$ \citep{ShimonyEtAl2022SOCS} since it can be efficiently computed using biconnected components as follows:
Let $G^+_r(G,N.\pi)= G_r(G,N.\pi)\cup (N.head,t)$, and let $B$ be the biconnected block of $G^+_r(G,N.\pi)$ containing $N.head$ and $t$. Then $v$ is a vertex
on some simple path  from $N.head$ to $t$ in $G_r(G,N.\pi)$ iff $v$ is in  $B$.
The number of edges in a simple path from $N.head$ to $t$ is thus at most $h_{BCC}=|V(B)|-1$,
where $V(B)$ are the vertices of $B$. A stronger heuristic  ($h_{MIS}$) based on maximum independent sets and triconnectivity exists 
\citep{ShimonyEtAl2022SOCS,DBLP:conf/socs/DahanTSD24}. As $h_{MIS}$ was ineffective for highly connected graphs, as used in our experiments, this heuristic is not further discussed here.
}

Snake problems have stronger heuristics~\citep{DBLP:conf/socs/PalomboSPFKR15other}.
For example,
the {\em snake head} heuristic ($h_{SH}$) 
recognizes that only one neighbor of the head can be used in a snake path. Another heuristic is the {\em star pattern} heuristic \citep{DBLP:conf/socs/PalomboSPFKR15other}, which subtracts 1 from the bound for each vertex and all its $(\geq 3$) neighbors not yet visited. This heuristic was improved by \citet{ShimonyEtAl2022SOCS} to
a vertex and only 3 of its neighbors, called the "{\em Y} heuristic ($h_Y$). We use the best-performing heuristic in our evaluation.

\subsubsection{Unidirectional DFBnB for GLSP}

Unidirectional Depth-First Branch and Bound (DFBnB) for longest path (denoted here by XDFBnB) \cite{DBLP:conf/socs/SternKPFR14} iteratively extends a single path starting from the initial vertex $s$. It maintains an incumbent longest $s$-$t$ path, $S$, found so far. At each step, an admissible heuristic upper bound on the remaining path length is added to the current path's length. If this sum is not strictly greater than $|S|$, the current branch cannot improve the incumbent and is safely pruned. Otherwise, the valid successors of the current path's head are generated, sorted in descending order of their path cost estimates $f=g+h$, and recursively expanded. This node ordering encourages finding long paths early, thus rapidly tightening the lower bound $|S|$ and pruning remaining branches more effectively.

\subsection{Bidirectional Search}



In \bihs for MIN problems, the objective is to find a least-cost path between \start\ and \goal\ in a given graph $G$. 
The shortest distance between nodes $x$ and $y$ is denoted by $\mathit{c}(x,y)$, where $\mathit{c}(\start, \goal) = \Cstar$. 
\bihs\ performs a forward search ($F$) from \start\ and a backward search ($B$) from \goal, continuing until the two search frontiers meet. \bihs algorithms typically maintain two open lists, \openf\ and \openb, for the forward and backward searches, respectively. Each node is associated with $g$-, $h$-, and $f$-values. For a given direction $D \in \{F,B\}$, these values are denoted as $\gD$, $\hD$, and $\fD$ (i.e., \(g_{F}, h_{F}, f_{F}\) for forward search and \(g_{B}, h_{B}, f_{B}\) for backward search). 

Most \bihs\ algorithms utilize two \emph{front-to-end} (F2E) heuristic functions~\citep{Kaindl1997}, $\hF(s)$ and $\hB(s)$, which estimate $\mathit{c}(s,\goal)$ and $\mathit{c}(\start,s)$ for all $s \in G$, respectively. A forward heuristic $\hF$ is said to be \emph{admissible} if $\hF(s) \leq \mathit{c}(s, \goal)$ for all $s$ in $G$, and \emph{consistent} if $\hF(s) \leq \mathit{c}(s, s') + \hF(s')$ for all $s$ and $s'$ in $G$. Analogous definitions apply for backward admissibility and consistency.
More informative are {\em front-to-front} (F2F) heuristics, which instead estimate the distance $c(s,s')$ between a given node $s$ on the forward frontier to each node $s'$ on the backward frontier by $h(s,s')$.
A typical priority function here is then
given by:
\[
f_F(s)=g_F(s)+\min_{s'\in O_B} (h(s,s')+g_B(s'))
\]


Since the heuristic estimations in this
$f_F(s)$ is for smaller parts (between two frontier nodes) it tends to be more accurate than front-to-end heuristics (which estimate all the way towards the goal).
Unfortunately, despite drastically decreasing the number of expansions,
the overhead due to having to compute the minimum over
all the nodes in the opposite frontier typically increases the overall runtime. Thus, front-to-front heuristics are rarely used in practice.

{\em Single frontier bidirectional search} (SFBDS)~\cite{SFBDS} is a general 
front-to-front \BiHS\ framework which only uses a single frontier. SFBDS basically spans a search tree that can then be searched with any known search algorithm. A node in a search
tree spanned by SFBDS consists of a pair of states $(x,y)$: a forward state $x$ and a backward state $y$. Such a node corresponds to
the task of finding a path between $x$ and $y$. This task is recursively
decomposed by expanding either $x$ or $y$ and generating new tasks between
either \CRC{R1MC2}{(1) $\Gamma(x)$ and $y$, or (2) $\Gamma(y)$ and $x$.}
At every node, \CRC{AllReviewers}{an {\em expansion policy} (known as {\em jumping policy})} decides which of the two states to expand
next, i.e., the search can proceed forward or backward.
Given a fixed expansion policy, a tree is induced where the cost of a node in that tree is $f(x,y)=\gF(x)+\gB(y) + h(x,y)$, where $h(x,y)$ is a F2F heuristic between $x$ and $y$. That tree 
can be searched using any admissible search algorithm, such as A* or IDA*.  


\subsection{Bidirectional Search in GLSP}

\bihs requires the ability to perform backward search from the goal state(s), as well as forward search from the initial state. However, in GLSP, a state consists of an entire path. 
To achieve backward search for GLSP, our backward search mirrors the forward search by growing an initially empty path starting from the target {\em vertex} $t$ rather than from a goal {\em state}.
Such backward search creates difficulties when aiming at bidirectional search - challenges that have been addressed by \citet{ShubiSFS25}, as follows.

\paragraph{Solution Detection.} In bidirectional search for GLSP, algorithms must cross-reference frontiers to find path intersections (e.g., by indexing nodes by their $head$ vertex). Once a match is found, the algorithm must perform \textit{path verification} to ensure the concatenated paths do not overlap or violate any GLSP constraints before a valid solution is declared.

\paragraph{Search Bounds.} When to halt the search is another non-trivial issue. In bidirectional best-first algorithms for MAX problems (like XMM), the search maintains a global upper bound on the longest path based on the heuristic evaluations of the forward and backward frontiers. The search proves optimality and halts when it finds a path whose cost equals this upper bound. In contrast, depth-first branch-and-bound approaches do not maintain expansive frontiers; instead, they maintain a global incumbent solution and aggressively prune any individual branch whose heuristic upper bound cannot strictly improve upon the incumbent.

\subsection{XMM and Bidirectional Baselines}

A prominent strategy in standard \bihs is attempting to "meet in the middle" by delaying the expansion of nodes that have progressed past the solution midpoint. For MAX problems like GLSP, the state-of-the-art baseline is {\em MM for Maximum}, \textbf{XMM}~\cite{ShubiSFS25}, which achieves this by ordering its open lists using the following priority function:
\begin{equation}\label{eq:MNM} 
pr_D(n) \triangleq \min(2h_D(n), g_D(n)+h_D(n))
\end{equation} 
This function behaves like standard $f=g+h$ initially, but safely delays the expansion of nodes where $g_D(n) > h_D(n)$ (nodes estimated to be past the middle). While the overestimating nature of MAX heuristics prevents XMM from strictly guaranteeing that no nodes beyond the midpoint are expanded, it remains the primary best-first bidirectional baseline for our empirical evaluation.

\section{Bidirectional DFBnB for GLSP}
\begin{algorithm}[h]
\caption{BiXDFBnB: Bidirectional DFBnB for Longest Simple Path}
\label{alg:BiXDFBnB}
\DontPrintSemicolon
\SetKwProg{Fn}{Function}{:}{}
\SetKwInOut{Input}{Input}
\SetKwInOut{Output}{Output}

\Input{Undirected graph $G=(V,E)$, start $s \in V$, target $t \in V$, admissible front-to-front heuristic $h_{F2F}$}
\Output{A longest simple $s$-$t$ path $S$}
\BlankLine

\textbf{global} $S \leftarrow \emptyset$\\
$\mathit{BiXDFBnB}(\mathit{MkNd}(s), \mathit{MkNd}(t))$\\
\Return $S$\\
\BlankLine

\Fn{$\mathit{BiXDFBnB}(s_F, s_B)$}{
    \If{$|s_F.\pi| + h_{F2F}(s_F, s_B) + |s_B.\pi| \le |S|$}{
        \Return \tcp*{Pruning}
    }
    
    \uIf{$s_F.\mathit{head} = s_B.\mathit{head}$}{
        $\pi \leftarrow s_F.\pi \cdot \overline{s_B.\pi}$ \tcp*{Exact Meet}
        \If{$|\pi| > |S|$}{
            $S \leftarrow \pi$\\
        }
        \Return\\
    }
    \uElseIf{$s_F.\mathit{head} \in s_B.\mathit{tail}$ \textbf{or} $s_B.\mathit{head} \in s_F.\mathit{tail}$}{
        \Return \tcp*{Overlap Rejection}
    }
    \ElseIf{$(s_F.\mathit{head}, s_B.\mathit{head}) \in E$}{
        $\pi \leftarrow s_F.\pi \cdot \overline{s_B.\pi}$ \tcp*{Adjacent Meet}
        \If{$|\pi| > |S|$}{
            $S \leftarrow \pi$\\
        }
    }
    \BlankLine
    
    \tcp*{Simultaneous Expansion}
    \ForEach{$(succ_F, succ_B) \in \mathit{Sorted}(\Gamma(s_F) \times \Gamma(s_B))$}{
        $\mathit{BiXDFBnB}(succ_F, succ_B)$ \tcp*{Empty product prunes dead-ends}
    }
}
\end{algorithm}


\CRC{R4C1}{To overcome the limitations of standard bidirectional architectures in maximization settings, we introduce BiXDFBnB. This algorithm represents a non-trivial realization of the Single-Frontier Bidirectional Search (SFBDS) framework~\cite{SFBDS} engineered specifically for the complex constraints of GLSP. While SFBDS was originally conceived to minimize frontier matching overhead in shortest-path (MIN) problems, transitioning it to a maximization (MAX) depth-first branch-and-bound framework requires a fundamental rethink of node representation, expansion coordination, and pruning mechanics. Rather than simply tracking independent frontiers, BiXDFBnB structuralizes the search space into tightly synchronized paired states, turning what is traditionally a costly frontier-coordination liability into an algorithmic advantage for tracking path constraints.}


\subsection{The BiXDFBnB Algorithm}

Algorithm \ref{alg:BiXDFBnB} outlines the pseudocode for BiXDFBnB. The algorithm maintains the incumbent longest path found so far, denoted by $S$, initialized to an empty path (line 1). Instead of maintaining expanding open lists, at each recursive step, the search evaluates a single pair of states: a forward state $s_F$ representing a path from $s$, and a backward state $s_B$ representing a path from $t$.
The recursive call has three primary phases: pruning, path overlap rejection, and expansion.

\subsubsection{Bound Pruning}

(Lines 5-6): The algorithm relies on an admissible front-to-front heuristic $h_{F2F}(s_F, s_B)$, which provides an upper bound on the length of the longest simple path connecting $s_F.head$ to $s_B.head$ using only remaining valid vertices. 
In Line 5, the algorithm evaluates the pruning condition. If the combined length of the forward path, the backward path, and the heuristic estimate is less than or equal to the length of the incumbent solution $|S|$, the current branch cannot possibly yield a strictly longer path. The algorithm thus safely prunes the node and backtracks.

\subsubsection{Pruning overlapping paths}

(Lines 12-13): To guarantee path compliance with the global exclusion constraint $L$, the algorithm verifies that the head of one state has not collided with the previously visited body ($tail$) or other constrained vertices of the opposing state. If such an intersection occurs, the paths cross illegally, and the branch is pruned.

\subsubsection{Node Expansion Strategy}

(Lines 18-19): The canonical dynamic of SFBDS alternates between expanding in one direction by a single vertex while freezing the other. However, BiXDFBnB opts for a \textit{Simultaneous Expansion} paradigm. Rather than choosing one side to advance, it expands both paths simultaneously by generating the Cartesian product of the valid successor sets: $\Gamma(s_F) \times \Gamma(s_B)$. For each pair $(succ_F, succ_B)$ generated by this Cartesian product, the algorithm computes the front-to-front heuristic $h_{F2F}(succ_F, succ_B)$. The pairs are sorted in descending order according to these heuristic estimates and explored via depth-first recursive calls. 

Due to this expansion scheme, there are two possible path meeting scenarios, as we next explain.

\subsubsection{Exact Meet}

(Lines 7-11): If both paths terminate on the exact same vertex, a valid $s$-$t$ path is formed by concatenating $s_F.\pi$ with the reversed $s_B.\pi$. If this new path improves upon $S$, the incumbent is updated. Because the paths have collided, no further simple extensions are possible, and the algorithm backtracks.

\subsubsection{Adjacent Meet}

(Lines 14-17): Advancing both paths simultaneously induces the \emph{Adjacent Meet} condition: the forward and backward path heads can become adjacent in $G$ before landing on the exact same vertex. 
In this scenario, a valid $s$-$t$ path is successfully bridged via the connecting edge, and the incumbent $S$ is updated if this bridged path is longer. Simultaneous expansion effectively collapses two independent alternating steps into one, tightly coupling the front-to-front heuristic evaluation of the joint moves while naturally catching valid paths before the search frontiers illegally cross.
After an adjacent meet the BiXDFBnB does not backtrack. Indeed, it prevents both searches from continuing via the crossing edge. However, because new longer paths can still be formed, the search continues via other edges.

\subsection{Generalizing to GLSP}
While Algorithm \ref{alg:BiXDFBnB} focuses on the Longest Simple Path (LSP) problem, BiXDFBnB naturally generalizes to all Generalized Longest Simple Path (GLSP) framework problems. Simply replacing  the overlap rejection conditions  with the appropriate GLSP constraint check suffices.
However, some optimizations become possible as well.
To demonstrate this adaptability, we use the Snake problem as a representative example, requiring only two minor adjustments:
\subsubsection{General Overlap Rejection} 
For LSP, paths must avoid the opposing visited body ($tail$). Because Snake requires an induced (chordless) path, the overlap rejection condition (Line 12) must be broadened. The head of one state must not land on any \textit{illegal} vertices of the opposing state, which includes both the opposing $tail$ and all of its neighbors.
\subsubsection{Immediate Backtracking on Adjacent Meets} 
In LSP, the search continues after an adjacent meet to find potentially longer detours. For Snake, however, the algorithm must backtrack immediately by inserting a \textbf{return} statement after line 17, since any further extensions would instantly yield an illegal chord. Therefore, after updating the incumbent, no strictly longer valid snake can exist from these prefixes, and the branch should be pruned.

\subsection{Theoretical Properties}
\label{subsec:theoretical_properties}

A crucial property of BiXDFBnB is that, because it uses a depth-first search strategy, its memory usage is \CRC{R1C1}{linear in the search depth $d$. Specifically, the memory complexity is $O(d \cdot b^2)$,  where $b$ is the maximum branching factor.} This provides a significant advantage over algorithms that maintain explicit open lists, whose memory requirements can scale exponentially.

Next, we establish the soundness, completeness, and optimality of BiXDFBnB, as well as its meet-in-the-middle property. A bidirectional algorithm is particularly vulnerable to generating cyclic paths or bypassing the optimal solution due to search frontiers crossing illegally. We show that the synchronized Cartesian expansion inherently resolves these parity issues.

\begin{lemma}[Soundness]\label{lem:soundness}
BiXDFBnB only generates valid, simple $s$-$t$ paths.
\end{lemma}
\begin{proof}
Candidate paths are only formed under two mutually exclusive conditions: Exact Meet and Adjacent Meet. 
Suppose a forward path $\pi_F$ and a backward path $\pi_B$ share a vertex $v$. Because the search expands both frontiers simultaneously via the Cartesian product, the paths grow at an identical rate: $|\pi_F| = |\pi_B| = k$ at any search depth $k$. Three distinct cases arise:

{\bf(1)} If both paths arrive at $v$ at the exact same depth $k$, the algorithm triggers an \textit{Exact Meet} ($s_F.head = s_B.head$), safely merges the paths, and backtracks, preventing further expansion. 

{\bf (2)}  If the frontiers try to cross each other by traversing the same edge in opposite directions simultaneously, they must have been positioned on adjacent vertices at the previous depth $k-1$. This scenario is perfectly intercepted by the \textit{Adjacent Meet} condition, avoiding illegal intersections.

{\bf (3)} Alternatively, if $\pi_F$ reaches $v$ at depth $k_1$ and $\pi_B$ reaches $v$ at depth $k_2$ where $k_1 < k_2$, then $v$ becomes part of $s_F.tail$. When the backward search lands on $v$ at depth $k_2$, the \textit{Overlap Rejection} condition ($s_B.head \in s_F.tail$) is triggered, immediately pruning the branch.

Consequently, it is topologically impossible to concatenate two overlapping paths, ensuring all generated paths are strictly simple. \end{proof}

\begin{lemma}[Intersection Guarantee]\label{lem:intersection}
For any valid simple $s$-$t$ path $P$, the simultaneous expansion strategy guarantees the search frontiers will trigger an Exact Meet or Adjacent Meet, provided the branches are not pruned.
\end{lemma}
\begin{proof}
Let $P$ be a simple $s$-$t$ path of length $L$ (consisting of $L$ edges). Because the Cartesian product $\Gamma(s_F) \times \Gamma(s_B)$ is explored at each step, the algorithm synchronously evaluates prefixes and suffixes of $P$, expanding by exactly one edge per prefix and one edge per suffix at a time.

\textbf{Even Length ($L = 2k$):} $P$ contains an exact middle vertex. At search depth $k$, both $s_F$ and $s_B$ will have traversed exactly $k$ edges, landing simultaneously on this middle vertex. This triggers the Exact Meet condition.
    
\textbf{Odd Length ($L = 2k + 1$):} $P$ does not have a middle vertex, but rather a middle edge. At search depth $k$, $s_F$ and $s_B$ will have traversed $k$ edges, landing on the respective endpoints of this middle edge. Because the graph $G$ is undirected, the bridging edge connecting them exists in $E$, triggering the Adjacent Meet condition.

Thus, parity cannot cause undetected crossings.
\end{proof}

\begin{corollary}
For state spaces with unit edge weights, the forward and backward search frontiers of BiXDFBnB intersect at a path length of $k$ or $k+1$ from their respective origin nodes, guaranteeing that the searches meet in the middle.
\end{corollary}
\begin{proof}
This follows directly from Lemmas 1 and 2.
\end{proof}

\begin{theorem}[Completeness and Optimality]
Given an admissible front-to-front heuristic $h_{F2F}$, BiXDFBnB is guaranteed to return a longest simple $s$-$t$ path.
\end{theorem}
\begin{proof}
By Lemmas \ref{lem:soundness} and \ref{lem:intersection}, the algorithm generates only valid paths and will correctly detect the optimal path $P^*$ (with maximum length $L^*$) if it is explored. It remains to show that prefixes of $P^*$ are never prematurely pruned.

A branch is pruned if and only if $|s_F.\pi| + h_{F2F}(s_F, s_B) + |s_B.\pi| \le |S|$. By definition, an admissible MAX heuristic guarantees that $h_{F2F}(s_F, s_B)$ is an upper bound on the length of any simple path connecting $s_F.head$ to $s_B.head$ using the remaining valid vertices. Therefore, the sum on the left side of the inequality represents an upper bound on the length of any path derivable from the current state pair. For $P^*$ to be pruned, this upper bound must be $\le |S|$, which implies $L^* \le |S|$. This condition is only met if the incumbent solution $S$ already has a length equal to or greater than that of $P^*$, guaranteeing that the search will ultimately return an optimal solution.
\end{proof}

\section{Analysis of BiXDFBnB}

\subsection{Relation to SFBDS and Expansion Policies}

Single-Frontier Bidirectional Search (SFBDS) departs from traditional bidirectional search by using a single frontier instead of two. Each node in the SFBDS search tree can be seen as an independent task of finding the shortest path between the current start and current goal. Likewise, our algorithm can be viewed as having a single frontier, where each node represents both a path from $s$ and a path from $t$, with the task of finding the longest simple path that has the former as a prefix and the latter as a suffix.

\CRC{AllReviews}{
One aim of SFBDS is to minimize search effort by choosing an appropriate {\em expansion policy} (AKA jumping policy), which decides which of two sides to expand next, choosing the direction with the highest potential for minimizing the total search effort.}
\CRC{AllReviews}{Three expansion policies are considered:}
\begin{itemize}
    
    \item \CRC{AllReviews}{\textbf{Alternating} (AKA {\em always-jump}): A simple baseline expansion policy that alternately expands the forward and the backward side. In domains with unit edge costs, like GLSP, this ensures that the search meets in the middle.}
    
    \item \CRC{AllReviews}{\textbf{Greedy} (AKA {\em 1-step lookahead}): Performs a 1-step lookahead by peeking at all the children of $x$ and measuring their heuristic towards $y$. Likewise, performs a 1-step lookahead by peeking at all the children of $y$ to measure their heuristic towards $x$. Finally, it  expands the side which contains the node with the largest heuristic value.}
    
    \item \CRC{AllReviews}{\textbf{Simultaneous} (AKA {\em 2-step lookahead}): Performs a 2-step lookahead by peeking at all the children of $x$ and all the children of $y$, and measuring their heuristic towards each other. Then, it expands the node that includes the child of $x$ and the child of $y$ with the largest heuristic value. This policy also ensures that the search meets in the middle in GLSP domains with unit edge costs.}
\end{itemize}


\subsubsection*{Theoretical Trade-offs}
\CRC{AllReviews}{
Each expansion policy balances node pruning against the computational overhead of heuristic evaluations, which scales with the branching factor ($b$). The \textbf{Alternating} policy is computationally cheap, requiring only $b$ heuristic computations per expansion. While it guarantees balanced search depth, its rigid expansion order means it cannot leverage heuristic guidance to dynamically prioritize the more promising frontier. The \textbf{Greedy} policy improves pruning by independently evaluating the immediate successors of both frontiers, requiring $2b$ heuristic computations per expansion. This local perspective helps bypass poor branches, though evaluating each frontier in isolation means it cannot account for the direct heuristic interactions between opposing successors. The \textbf{Simultaneous} policy maximizes pruning power by evaluating the Cartesian product of all forward and backward successors, steering the search toward the best front-to-front connections. It requires $b^2$ heuristic computations per expansion---a quadratic overhead that risks outweighing the runtime gains from reduced node generation.}

\subsubsection*{Empirical Comparison}
\CRC{AllReviews}{
We conducted preliminary experiments testing each expansion policy on two domains: LSP in Grids and Coil-in-the-Box (CIB). Results, presented in Tables \ref{tab:jumping_policies_lsp} and \ref{tab:jumping_policies_cib}, show that the \textbf{Simultaneous} expansion policy usually outperforms the alternatives. Despite its quadratic per-expansion overhead, its superior pruning power drastically reduces overall node expansions. Across both domains, this reduction consistently outweighs the higher evaluation cost, yielding the best runtimes and solving capabilities in nearly all tested configurations.}
\addtolength{\tabcolsep}{-0.138cm}

\begin{table}[h]
\centering
\begin{tabular}{llrrrr}
\hline

& & \multicolumn{2}{c}{20\%} & \multicolumn{2}{c}{12\%} \\ \cline{3-6}
& & Exp & Time [ms] & Exp & Time [ms] \\ \hline
\multirow{3}{*}{6x6} & Alternating & 107 & 52 & 252 & 143 \\
& Greedy & 170 & 190 & 417 & 518 \\
& Simultaneous & \textbf{52} & \textbf{36} & \textbf{132} & \textbf{97} \\ \hline
\multirow{3}{*}{7x7} & Alternating & 167 & 133 & 2,694 & 1,973 \\
& Greedy & 347 & 544 & 5,557 & 8,601 \\
& Simultaneous & \textbf{81} & \textbf{82} & \textbf{1,401} & \textbf{1,278} \\ \hline
\multirow{3}{*}{8x8} & Alternating & 5,102 & 3,972 & 78,490 & 71,793 \\
& Greedy & 10,632 & 18,258 & 156,085 & 278,222 \\
& Simultaneous & \textbf{2,707} & \textbf{2,470} & \textbf{39,373} & \textbf{44,301} \\ \hline
\end{tabular}
\caption{BiXDFBnB Jumping Policies on LSP in Grids}
\label{tab:jumping_policies_lsp}
\end{table}

\addtolength{\tabcolsep}{0.03cm}

\begin{table}[h]
  \centering
\begin{tabular}{llrrrr}
\hline
        & & 
\multicolumn{2}{c}{Finding} & \multicolumn{2}{c}{Proving} \\
\cline{3-6}
& & Exp & Time & Exp & Time \\
 \hline
\multirow{3}{*}{4D} & Alternating       & 3 & 2 & 3 & 2 \\
& Greedy  & 3 & 2 & 3 & 2 \\
& Simultaneous  & \textbf{2} & \textbf{1} & \textbf{2} & \textbf{1} \\
        \hline
\multirow{3}{*}{5D} & Alternating       & 9 & 16 & 17 & 28 \\
& Greedy  & 11 & 35 & 20 & 68 \\
& Simultaneous  & \textbf{5} & \textbf{11} & \textbf{9} & \textbf{26} \\
        \hline
\multirow{3}{*}{6D} & Alternating       & 788 & \textbf{2,462} & 3,748 & \textbf{14,998} \\
 & Greedy  & 722 & 2,551 & 6,703 & 45,673 \\
& Simultaneous  & \textbf{512} & 3,228 & \textbf{2,190} & 18,729\\
        \hline
\multirow{3}{*}{7D} & Alternating       & 14,543 & 125,952 & - & TO \\
 & Greedy  & 2,595,312 & 52,053,185 & - & TO \\
& Simultaneous  & \textbf{3,265} & \textbf{56,960} & - & TO \\
        \hline
    \end{tabular}
    \caption{BiXDFBnB Jumping Policies on Coil-in-the-Box}
    \label{tab:jumping_policies_cib}
\end{table}

\CRC{AllReviewers}{Given its expansions and runtime advantage, we adopt the \textbf{Simultaneous} expansion policy 
in BiXDFBnB experiments henceforth. Beyond the runtime benefits of massive node pruning, expanding both sides synchronously provides crucial algorithmic guarantees. As formally proven in Section \ref{subsec:theoretical_properties} (Theoretical Properties), this synchronized topological growth rigorously ensures that the optimal path will be intercepted - either via an exact vertex meet or an adjacent edge meet - effectively bypassing the parity vulnerabilities common in traditional bidirectional search. While our preliminary results show this policy dominates in GLSP and CIB, future work will explore whether these performance trends hold across a broader variety of domains and heuristic types.}


\subsection{F2F Heuristics in Bidirectional GLSP}
\CRC{R4C2}{
In bidirectional search for GLSP, a search state is defined as a pair $(s_F, s_B)$, where $s_F$ is the forward path starting from $s$, and $s_B$ is the backward path starting from $t$. The search progresses by extending either $s_F$ or $s_B$. In this context, the {\em remaining graph} $G_r(G, s_F.\pi \cup s_B.\pi)$ is defined as $G$ after removing (1) all vertices present in $s_F.tail \cup s_B.tail$, and (2) all other elements no longer allowed under the GLSP constraint $L$. Clearly, a paired state $(s_F, s_B)$ is not a dead end only if $s_F.head$ and $s_B.head$ are in the same connected component of $G_r(G, s_F.\pi \cup s_B.\pi)$. A F2F heuristic $h$ is said to be {\em admissible} for bidirectional GLSP iff for every paired state $(s_F, s_B)$ in the search space, $h(s_F, s_B)$ is {\em greater than or equal to} the weight of the longest constrained path (longest simple path in LSP) connecting $s_F.head$ to $s_B.head$ in $G_r(G, s_F.\pi \cup s_B.\pi)$. Crucially, because F2F heuristics evaluate the gap between the two frontiers, they can directly adapt existing F2E heuristics by treating the backward head, $s_B.head$, as if it were the fixed target $t$. 
}

\subsection{Examination of F2E vs. F2F}
Here we examine the behavior of front-to-end (F2E) and front-to-front (F2F) heuristics within the context of BiXDFBnB. We argue that F2F is naturally suited to this paired-state algorithmic structure (akin to SFBDS), whereas F2E proves inadequate. 
We subsequently compare this bidirectional approach to unidirectional DFBnB, supporting our claims with the experimental results.

\subsubsection{Front-to-Front Heuristics} There is a major structural difference between our DFBnB algorithm and a standard best-first bidirectional algorithm with two distinct open lists. With two distinct open lists, the search independently grows two frontiers towards opposing goals. When frontiers intersect, a candidate path is formed (in GLSP we still need to perform a validation check, as detailed above). 
A F2F heuristic in that context acts as a {\em more-informed F2E heuristic} by calculating estimates between the current state and \textit{all} states in the opposite frontier, selecting the best bound. 
While this provides stronger bounds than a F2E heuristic, the computational overhead of iterating over the entire opposing frontier typically outweighs the pruning benefits.

In contrast, the BiXDFBnB search operates on precisely one pair of states at a time, striving to bridge the gap between them. This affords F2F heuristics two decisive advantages. First, the distance calculation only needs to be executed between the two current heads, eliminating the quadratic algorithmic overhead. Second, the F2F heuristic calculates exactly what the paired search requires - a tight bound on the remaining path connecting the two specific states.

\subsubsection{Front-to-End Heuristics} In the context of our BiXDFBnB, using a F2E heuristic carries significant drawbacks. First, one must compute two independent heuristics (forward to $t$, backward to $s$) and aggregate them (taking the best bound), which is computationally wasteful. More fundamentally, these estimates ignore the constraints introduced by the opposing path blockages.

\smallskip
We can demonstrate the advantage of F2F heuristic over F2E heuristic in a DFBnB by looking at the remaining-graph heuristic ($G_r$). Suppose the forward and backward paths each have a $g$-value of 5. Using the F2E heuristic yields two independent bound calculations, where each remaining graph is constrained by only 5 visited vertices. Conversely, the true F2F heuristic evaluates a single remaining graph simultaneously constrained by all 10 visited vertices. Extracting a bound from a single, more severely depleted topology yields a consistently tighter admissible upper bound than taking the best of two less depleted topologies.

In unidirectional XDFBnB, only a single path is extended to the target, so the F2E and F2F heuristics converge by definition. For a mathematically fair comparison to a F2F heuristic where both paths have a length of 5, consider a unidirectional search that has reached depth 10. The remaining graph likewise drops 10 vertices. Despite matching vertex reduction, the bidirectional F2F heuristic frequently generates stronger bounds than the single-front heuristic. The bidirectional search isolates two distinct fragments of the environment, resulting in a significantly more fragmented remaining graph topology than a single, continuous 10-step path localized to one side. Given that heuristics like $h_{BCC}$ heavily exploit graph fragmentation, F2F pruning becomes substantially harder to overcome.


The empirical evaluation (Section \ref{sec:emp}) confirms the theoretical advantages of F2F over F2E in BiXDFBnB, as well as BiXDFBnB's merits relative to unidirectional XDFBnB and other state-of-the-art approaches. 
Across both the LSP Grid and Maze domains (Tables 1 and 2), the BiXDFBnB F2F algorithm required orders of magnitude fewer node expansions and drastically lower wall-clock times compared to BiXDFBnB F2E. Additionally, BiXDFBnB F2F dominates XDFBnB in most metrics across all instances. This supports our reasoning that iteratively growing two opposing paths rather than extending a single long path breaks apart the remaining graph topology far more intensely, triggering earlier and more effective heuristic cutoffs.

\section{Empirical Evaluation}
\label{sec:emp}

\paragraph{Goals.}
Our evaluation addresses three main questions:
(1) Does the use of a front-to-front (F2F) heuristic reduce node expansions compared to a front-to-end (F2E) heuristic in simultaneous bidirectional DFBnB?
(2) Do reductions in node expansions translate into runtime improvements?
(3) How does our approach compare to strong unidirectional and bidirectional baselines?

\paragraph{Domains.}
The evaluation includes the following domains:

\noindent \textbf{(1) LSP in Grids}. Here we find a longest simple path from $s$ (top left corner) to $t$ (bottom right corner) of a 2D grid. Grid sizes were all integer size combinations between $6\times6$ and $8\times8$ (a total of 6 combinations). For each size, we randomly set 8\%, 12\%, 16\%, and 20\% of the cells as obstacles, creating 10 random problem instances for each density (yielding 240 instances total). The heuristic used was $h_{BCC}$. 

\noindent \textbf{(2) Snake in Grids.} In this domain, we search for the longest snake in grids spanning all integer size combinations from $7\times7$ to $9\times9$ (6 combinations). The random instance-generation method mirrored the one used for LSP in grids, yielding another 240 total instances. For this domain, a combination of the $h_{BCC}$, $h_{SH}$, and $h_{Y}$ heuristics was utilized.

\noindent \textbf{(3) LSP in Mazes.} \CRC{Mazes}{This domain shares the same setup as LSP in grids, except that the initial grid is the original maze used by \citet{DBLP:conf/socs/DahanTSD24}, which serves as our single baseline instance (0 open diamonds). We generated increasingly difficult configurations by clearing randomly placed diamond-shaped areas within this base maze (see Figure \ref{fig:Diamond}). These areas are devoid of obstacles and create massive local state spaces. Specifically, we evaluated 10 random instances with 1 open diamond and 10 random instances with 2 open diamonds, utilizing the $h_{BCC}$ heuristic here as well.}

\noindent \textbf{(4) Coil-in-the-Box (CIB).} Here, we find a coil-in-the-box simple cycle in a hypercube, with experiments spanning 4D, 5D, 6D, and 7D. Because the first four steps in SIB and CIB are forced due to symmetry, they can be safely assumed and dropped from the search. Consequently, we define our start and target vertices as $s=(0,\dots,1,1,1,1)$ and $t=(0,\dots,0,0,0,0)$. Heuristics were again a combination of $h_{BCC}$, $h_{SH}$, and $h_{Y}$.

\paragraph{Algorithms.}
Compared algorithms are: unidirectional A*, bidirectional best-first search (XMM), unidirectional depth-first branch-and-bound (XDFBnB), and BiXDFBnB with both front-to-end (F2E) and front-to-front (F2F) heuristics. 

As unidirectional search performance can be highly asymmetric in some graph topologies, we execute all unidirectional baselines in both forward ($s \rightarrow t$) and backward ($t \rightarrow s$) directions. Unless explicitly separated in  data tables (as  in the Maze domain), the reported values for unidirectional algorithms are the best performance out of the two directions, ensuring the strongest  baseline comparison.

\paragraph{Experimental Setup.}
The evaluation was conducted using a 92-machine cluster, where
each machine possesses an AMD EPYC 7763 CPU with 256 hyper-threads and 1000GB of RAM. Each combination of search algorithm
and problem instance was allocated a single thread and a memory
limit of 60GB.
All algorithms were implemented in Python. 
This section highlights the main trends and representative results. Complete results, including both node expansions and runtimes for all instances appear in \cite{fullpaper}. The code and data are publicly available \cite{code}.

\subsection{Impact of Front-to-Front Heuristics}
The central claim of this paper - that F2F heuristics are vastly superior to F2E heuristics in the context of simultaneous DFBnB - is definitively validated across all domains. Table \ref{tab:f2e_comparison} presents a representative highlight of these results. As shown, \textbf{BiXDFBnB F2E} struggles significantly in environments with large open spaces, performing worse than or marginally comparable to traditional baselines. However, upgrading the heuristic to \textbf{F2F} reduces expansions by up to two orders of magnitude. Because the simultaneous expansion policy collapses the search to exactly one active pair of states, the quadratic overhead traditionally associated with F2F calculations is eliminated, translating this massive reduction in expansions directly into faster runtimes. 
The complete dataset comparing F2E and F2F across all tested domains and configurations appears in 
\cite{fullpaper}.

\begin{table}[t]
    \centering
    \small
    \setlength{\tabcolsep}{4pt}
    \begin{tabular}{@{}lrr|rr@{}}
        \toprule
        & \multicolumn{2}{c}{\textbf{Expansions}} & \multicolumn{2}{c}{\textbf{Runtime[ms]}} \\
        \cmidrule(lr){2-3} \cmidrule(lr){4-5}
        \textbf{Instance} & \textbf{F2E} & \textbf{F2F} & \textbf{F2E} & \textbf{F2F} \\
        \midrule
        Grid LSP $7\times8$ (12\%) &  165,241 &  \textbf{10,164} & 370,323 & \textbf{9,814} \\
        Grid Snake $7\times8$ (12\%)  &  7,430 & \textbf{299} & 59,729 & \textbf{1,528} \\
        Maze $od=1$             & 208,794 & \textbf{3,279} & 1,031,820 & \textbf{6,603} \\
        CiB 6D             & 163,561 & \textbf{2,190} & 1,375,600 & \textbf{18,729} \\
        \bottomrule
    \end{tabular}
    \caption{Comparison between F2E and F2F variants of BiXDFBnB in terms of node expansions and runtime.}\label{tab:f2e_comparison}
\end{table}

\subsection{Scalability in High-Dimensional CIB}
The Coil-in-the-Box (CIB) problem - finding the longest constrained cycle in an $n$-dimensional hypercube - serves as a pure test of exponential scalability. We evaluate the algorithms on dimensions 4D through 7D. 

A critical consideration in F2F heuristics is the trade-off between the time spent calculating a more informed heuristic and the time saved by expanding fewer nodes. Figure \ref{fig:cib_comparison} illustrates this relationship. While the F2F calculation is slightly more expensive per node, the exponential reduction in the number of nodes expanded leads to a drastic improvement in overall wall-clock time. In the 7D hypercube, BiXDFBnB F2F finds the optimal length in only \textbf{56 seconds}, whereas the F2E variant and A* fail to complete within a 12-hour time limit. The full numerical results for all CIB dimensions are provided in \cite{fullpaper}.

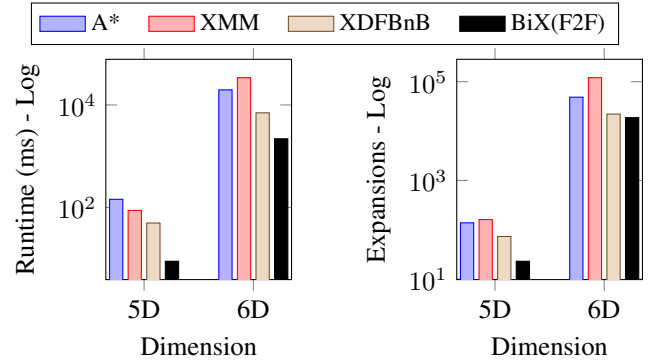
\begin{figure}[t]
    \centering
    \ref{shared_cib_legend}
    \vspace{2mm}
    \begin{tikzpicture}
        \begin{axis}[
            ybar, ymode=log,
            symbolic x coords={5D, 6D}, xtick=data,
            ylabel={Runtime (ms) - Log},
            xlabel={Dimension},
            bar width=5pt, width=0.48\columnwidth, height=4.5cm,
            enlarge x limits=0.35,
            area legend,
            legend columns=-1, 
            legend style={font=\footnotesize, /tikz/every even column/.append style={column sep=0.3cm}},
            legend to name=shared_cib_legend,
        ]
        \addplot coordinates {(5D,145) (6D,19669)}; 
        \addplot coordinates {(5D,88) (6D,34001)}; 
        \addplot coordinates {(5D,50) (6D,7018)};  
        \addplot[fill=black] coordinates {(5D,9) (6D,2190)}; 
        \legend{A*, XMM, XDFBnB, BiX(F2F)}
        \end{axis}
    \end{tikzpicture}
    \hfill
    \begin{tikzpicture}
        \begin{axis}[
            ybar, ymode=log,
            symbolic x coords={5D, 6D}, xtick=data,
            ylabel={Expansions - Log},
            xlabel={Dimension},
            bar width=5pt, width=0.48\columnwidth, height=4.5cm,
            enlarge x limits=0.35,
        ]
        \addplot coordinates {(5D,138) (6D,48625)}; 
        \addplot coordinates {(5D,161) (6D,120859)}; 
        \addplot coordinates {(5D,73) (6D,22085)};  
        \addplot[fill=black] coordinates {(5D,23) (6D,18729)}; 
        \end{axis}
    \end{tikzpicture}
\caption{Coil-in-the-Box results. \CRC{R1MC3}{F2F's expansion reduction (right) directly enables its runtime advantage (left). }}
    \label{fig:cib_comparison}
\end{figure}


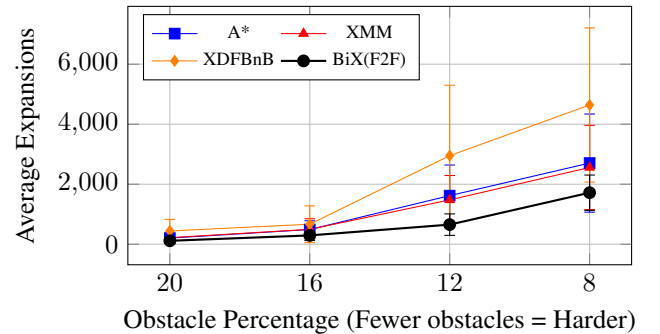
\begin{figure}[ht]
    \centering
    \begin{tikzpicture}
        \begin{axis}[
            xlabel={Obstacle Percentage (Fewer obstacles = Harder)},
            ylabel={Average Expansions},
            xtick={8, 12, 16, 20},
            x dir=reverse,
            legend pos=north west, 
            legend style={font=\scriptsize,legend columns=2},
            width=0.98\columnwidth,
            height=5cm,
            grid=major
        ]
        
        \addplot[mark=square*, blue, error bars/.cd, y dir=both, y explicit] 
            coordinates {
                (20, 204) +- (0, 92) 
                (16, 490) +- (0, 290) 
                (12, 1619) +- (0, 1020) 
                (8, 2704) +- (0, 1635)
            };
            
        \addplot[mark=triangle*, red, error bars/.cd, y dir=both, y explicit] 
            coordinates {
                (20, 218) +- (0, 96) 
                (16, 495) +- (0, 353) 
                (12, 1477) +- (0, 813) 
                (8, 2558) +- (0, 1404)
            };
            
        \addplot[mark=diamond*, orange, error bars/.cd, y dir=both, y explicit] 
            coordinates {
                (20, 442) +- (0, 385) 
                (16, 666) +- (0, 613) 
                (12, 2950) +- (0, 2348) 
                (8, 4640) +- (0, 2570)
            };
            
        \addplot[mark=*, black, thick, error bars/.cd, y dir=both, y explicit] 
            coordinates {
                (20, 114) +- (0, 81) 
                (16, 293) +- (0, 171) 
                (12, 653) +- (0, 358) 
                (8, 1717) +- (0, 588)
            };
            
        \legend{A*, XMM, XDFBnB, BiX(F2F)}
        \end{axis}
    \end{tikzpicture}
    \caption{Average and standard deviation of expansions for Snake in $8\times8$ Grids with varying obstacle densities. 
    }
    \label{fig:grid_expansions}
\end{figure}

\begin{table*}[t] 
    \centering
    
    \begin{minipage}{0.7\textwidth} 
        \centering
        \small
        \begin{tabular}{l|rr|rr|rr}
        \hline
        \textbf{Algorithm} & \multicolumn{2}{c|}{\textbf{0 Open Diamonds}} & \multicolumn{2}{c|}{\textbf{1 Open Diamond}} & \multicolumn{2}{c}{\textbf{2 Open Diamonds}} \\
        \cline{2-7}
        \textbf{} & Expansions & Time[ms] & Expansions & Time[ms] & Expansions & Time[ms] \\
        \hline
        \textbf{A* s$\rightarrow$t} & 1,162 & 2,934 & 30,949 & 77,636 & OOM & OOM \\
        \textbf{A* t$\rightarrow$s} & 657 & 1,444 & 55,400 & 119,738 & OOM & OOM \\
        \textbf{XMM} & 1,131 & 3,377 & 59,903 & 163,786 & OOM & OOM \\
        \textbf{XDFBnB s$\rightarrow$t} & 1,302 & 2,423 & 44,473 & 104,125 & 1,470,196 & 3,540,864 \\
        \textbf{XDFBnB t$\rightarrow$s} & 888 & 1,670 & 89,202 & 184,637 & 3,288,267 & 7,379,854 \\
        \textbf{BiXDFBnB F2E} & 4,606 & 19,859 & 515,170 & 2,461,587 & TO & TO \\
        \textbf{BiXDFBnB F2F} & \textbf{501} & \textbf{948} & \textbf{12,817} & \textbf{27,751} & \textbf{187,759} & \textbf{433,233} \\
        \hline
        \end{tabular}
        \caption{LSP in Mazes results.}
        \label{tab:LSPmazes}
    \end{minipage}\hfill 
    \begin{minipage}{0.25\textwidth} 
        \centering
        \includegraphics[width=0.8\linewidth]{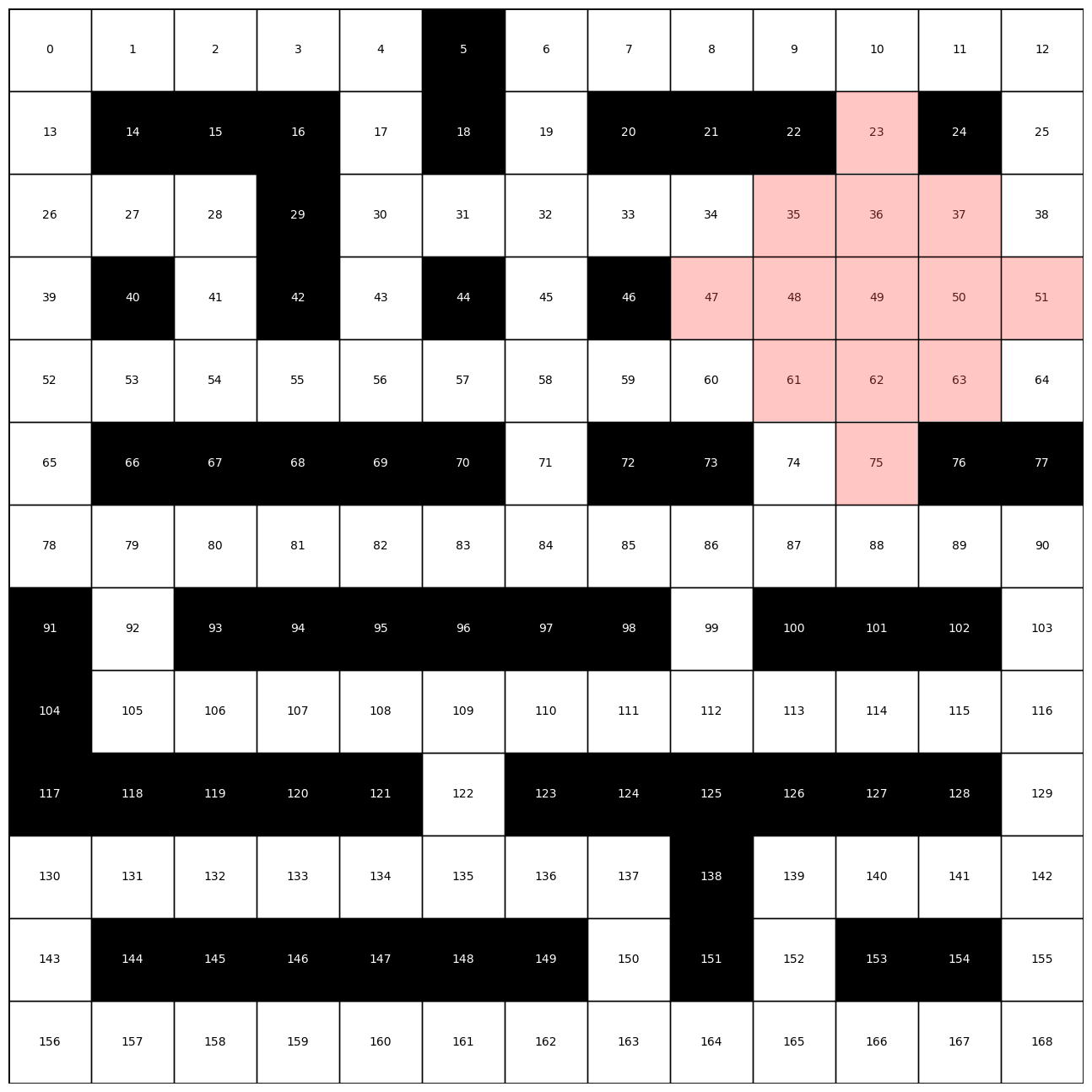}
        \captionof{figure}{Maze + 1 Diamond} 
        \label{fig:Diamond}
    \end{minipage}
    
\end{table*}

\subsection{Performance on Grids and Mazes}

For 2D environments, we evaluated both the standard Longest Simple Path (LSP) and the Snake problem (which strictly requires an induced path with no adjacent shortcuts) across square and rectangular grids seeded with varying percentages of obstacles ($20\%, 16\%, 12\%, 8\%$). As the number of obstacles decreases, the branching factor increases, making the state-space significantly harder to search.

\textbf{LSP in Grids:} Across all tested LSP grid instances, BiXDFBnB F2F consistently outperformed BiXDFBnB F2E, XDFBnB and the best-first baselines in both expansions and runtime. The complete tabular results for the LSP grid experiments can be found in \cite{fullpaper}.

\textbf{Snake in Grids:} The added induced-path constraints of the Snake variant naturally restrict the overall state space compared to pure LSP, but the relative performance gap between F2F and the baselines remains consistent across both problem types. Figure \ref{fig:grid_expansions} illustrates the representative node expansions on $8\times8$ grids for the Snake problem. BiXDFBnB F2F dominates across all densities,
leveraging its synchronized F2F pruning to significantly reduce the runtime.


\textbf{LSP in Mazes:} 
\CRC{Mazes}{
As shown in Table \ref{tab:LSPmazes}, our proposed algorithm dominates all other algorithms in both expansions and average runtime. The large state space in the 2 Open Diamonds instances caused the best-first algorithms A* and XMM to run out of memory and BiXDFBnB F2E to time out. BiXDFBnB F2F outperforms even the best-performing unidirectional search, reducing average expansions by a factor of four and runtime by nearly an order of magnitude. 
To determine the statistical significance of the performance differences, we conducted a one-sided Wilcoxon signed-rank test ($\alpha = 0.05$). This non-parametric method was selected to account for the heavily skewed distributions typical of search algorithm metrics. Across both the 1-Open and 2-Open Diamond configurations, the bidirectional algorithm yielded statistically significant improvements over both the $s \to t$ and $t \to s$ unidirectional variants in both total expansions and execution time. With all test pairings resulting in $p$-values below $0.005$ (Wilcoxon test statistics $W \le 3.0$), we confidently reject the null hypothesis, confirming that bidirectional search requires significantly fewer expansions and less runtime in these environments.
}

\subsection{Summary of Results}
Across all empirical domains, BiXDFBnB F2F emerges as the state-of-the-art solver for GLSP. By leveraging a depth-first branch-and-bound framework, it eliminates the exponential memory constraints of A* and XMM. Furthermore, by expanding nodes simultaneously, it mathematically guarantees path parity while unlocking the massive pruning power of front-to-front heuristics. This unique combination reduces node expansions by up to two orders of magnitude and translates directly into superior overall runtimes, rendering the traditional F2F computational overhead negligible.

\section{Conclusion and Future Work}

In this paper, we introduced BiXDFBnB, a novel bidirectional depth-first branch-and-bound algorithm designed for the Generalized Longest Simple Path (GLSP) problem. By evaluating opposing bounds simultaneously, our approach successfully integrates front-to-front (F2F) heuristics into a single-frontier search framework. While F2F heuristics are typically avoided in standard bidirectional search due to prohibitive computational overhead, our empirical evaluation demonstrates that in the context of GLSP - where state representation and path verification are inherently costly - F2F evaluations become not only viable but highly advantageous. BiXDFBnB effectively mitigates exponential memory bottlenecks that plague best-first algorithms such as A* and XMM, usually achieving reductions in both node expansions and overall runtime across Grids, Mazes, Snake-in-the-Box, and Coil-in-the-Box (CIB) domains.

A highly promising future research direction is  exploring dynamic jumping policies. Our current implementation relies on a static, alternating policy to balance the forward and backward search frontiers. However, preliminary evaluations using 1-step and double-lookahead policies suggest that dynamically selecting the expansion direction based on maximum pruning potential can significantly reduce node expansions. Fully integrating and optimizing these dynamic lookahead policies - without incurring the prohibitive runtime overhead observed in higher dimensions - remains an open challenge. Importantly, this direction is also highly relevant to Single-Frontier Bidirectional Search (SFBDS) in minimization problems, providing a broader context to explore advanced expansion strategies. Additionally, adapting the BiXDFBnB F2F framework to parallel processing architectures or extending it to other constrained maximization problems presents an exciting opportunity for further study.


\section*{Acknowledgments}
Supported by the Israel Science Foundation (ISF)
grant \#909/23 and by the Frankel center for CS at BGU.
\bibliography{LSP}

\clearpage
\appendix
\onecolumn
\section{Supplementary Materials: Full Results Data}
\label{appendix:A}

This appendix contains the comprehensive, instance-by-instance empirical results that summarize the aggregated plots and highlights discussed in Section \ref{sec:emp}. The tables below provide the absolute runtime (in milliseconds) and total node expansions required to prove the optimal solution across all tested algorithms. 

\subsection{Longest Simple Path in Grids}
The following table details the performance across grid dimensions under varying obstacle densities (8\%, 12\%, 16\%, 20\%). 

\begin{table*}[ht!]
\centering
\resizebox{0.9\textwidth}{!}{ 

\begin{tabular}{c|l|rr|rr|rr|rr}
\hline
\textbf{Size} & \textbf{Algorithm} & \multicolumn{2}{c|}{\textbf{20\%}} & \multicolumn{2}{c|}{\textbf{16\%}} & \multicolumn{2}{c|}{\textbf{12\%}} & \multicolumn{2}{c}{\textbf{8\%}} \\ \cline{3-10}
\textbf{} & \textbf{} & Expansions & Time[ms] & Expansions & Time[ms] & Expansions & Time[ms] & Expansions & Time[ms] \\ \hline
\multirow{5}{*}{\textbf{6x6}} & \textbf{A*} & 138 & 119 & 238 & 223 & 342 & 326 & 1,315 & 1,545 \\
& \textbf{XMM} & 81 & 94 & 128 & 159 & 241 & 271 & \textbf{701} & 894 \\
& \textbf{XDFBnB} & 97 & 51 & 174 & 91 & 194 & 124 & 1,168 & \textbf{660} \\
& \textbf{BiX-DFBnB F2E} & 308 & 430 & 368 & 578 & 498 & 757 & 5,130 & 8,154 \\
& \textbf{BiX-DFBnB F2F} & \textbf{52} & \textbf{36} & \textbf{91} & \textbf{60} & \textbf{132} & \textbf{97} & 1,007 & 673 \\ \hline
\multirow{5}{*}{\textbf{6x7}} & \textbf{A*} & 140 & 140 & 317 & 344 & 566 & 630 & 4,264 & 6,237 \\
& \textbf{XMM} & 100 & 124 & 264 & 335 & 454 & 580 & \textbf{2,052} & \textbf{3,277} \\
& \textbf{XDFBnB} & 84 & 64 & 162 & 114 & 340 & 235 & 7,532 & 4,520 \\
& \textbf{BiX-DFBnB F2E} & 385 & 588 & 419 & 704 & 1,262 & 2,286 & 63,120 & 121,083 \\
& \textbf{BiX-DFBnB F2F} & \textbf{64} & \textbf{53} & \textbf{97} & \textbf{81} & \textbf{221} & \textbf{208} & 4,050 & 3,400 \\ \hline
\multirow{5}{*}{\textbf{6x8}} & \textbf{A*} & 260 & 300 & 1,037 & 1,310 & 1,397 & 1,909 & 15,027 & 24,037 \\
& \textbf{XMM} & 124 & 170 & \textbf{728} & 1,165 & 848 & 1,290 & \textbf{6216} & \textbf{11,306} \\
& \textbf{XDFBnB} & 129 & 93 & 978 & 666 & 1216 & 860 & 35,500 & 22,108 \\
& \textbf{BiX-DFBnB F2E} & 267 & 457 & 5,397 & 9,766 & 2,047 & 4,083 & 308,380 & 637,224 \\
& \textbf{BiX-DFBnB F2F} & \textbf{68} & \textbf{58} & 844 & \textbf{634} & \textbf{562} & \textbf{539} & 18,200 & 15,417 \\ \hline
\multirow{5}{*}{\textbf{7x7}} & \textbf{A*} & 255 & 278 & 1,191 & 1,612 & 1,990 & 2,865 & 10,694 & 19,289 \\
& \textbf{XMM} & 177 & 252 & \textbf{587} & 916 & 1,448 & 2,498 & \textbf{5,192} & \textbf{10,400} \\
& \textbf{XDFBnB} & 154 & 106 & 1,058 & 734 & 1,661 & \textbf{1,230} & 26,056 & 17,623 \\
& \textbf{BiX-DFBnB F2E} & 304 & 537 & 5,070 & 9,553 & 10,301 & 20,962 & 194,238 & 411,556 \\
& \textbf{BiX-DFBnB F2F} & \textbf{81} & \textbf{82} & 719 & \textbf{630} & \textbf{1,401} & 1,278 & 13,492 & 12,166 \\ \hline
\multirow{5}{*}{\textbf{7x8}} & \textbf{A*} & 444 & 598 & 2,398 & 3,653 & 10,052 & 17,987 & 18,797 & 38,451 \\
& \textbf{XMM} & 290 & 448 & 1,966 & 3,526 & \textbf{6,658} & 13,512 & \textbf{11,738} & \textbf{26,391} \\
& \textbf{XDFBnB} & 292 & 255 & \textbf{1,185} & \textbf{996} & 16,928 & 13,019 & 75,606 & 56,350 \\
& \textbf{BiX-DFBnB F2E} & 1,020 & 2,107 & 9,722 & 20,825 & 165,241 & 370,323 & 788,754 & 1,810,950 \\
& \textbf{BiX-DFBnB F2F} & \textbf{187} & \textbf{204} & 1,283 & 1,311 & 10,164 & \textbf{9,814} & 45,628 & 43,688 \\ \hline
\multirow{5}{*}{\textbf{8x8}} & \textbf{A*} & 3,915 & 6,794 & 5,668 & 10,021 & 30,562 & 59,140 & 103,223 & 235,225 \\
& \textbf{XMM} & \textbf{1,484} & 2,920 & \textbf{2,158} & \textbf{4,077} & \textbf{11,247} & \textbf{24,951} & \textbf{51,885} & \textbf{134,016} \\
& \textbf{XDFBnB} & 2,364 & \textbf{1,992} & 6,525 & 5,524 & 47,121 & 41,467 & 267,658 & 280,918 \\
& \textbf{BiX-DFBnB F2E} & 27,536 & 59,561 & 62,407 & 150,406 & 991,999 & 2,388,260 & 5,255,810 & 13,219,500 \\
& \textbf{BiX-DFBnB F2F} & 2,707 & 2,470 & 4,226 & 4,470 & 39,373 & 44,301 & 225,468 & 242,148 \\ \hline
\end{tabular}
}
\caption{LSP in Grids results.}\label{tab:LSPgrids}
\end{table*}

\subsection{Snake in Grids}
The following table details the performance for the snake-in-the-box problem in grids under varying obstacle densities (8\%, 12\%, 16\%, 20\%). 
\begin{table*}[h]
\centering
\resizebox{0.9\textwidth}{!}{ 

\begin{tabular}{c|l|rr|rr|rr|rr}
\hline
\textbf{Size} & \textbf{Algorithm} & \multicolumn{2}{c|}{\textbf{20\%}} & \multicolumn{2}{c|}{\textbf{16\%}} & \multicolumn{2}{c|}{\textbf{12\%}} & \multicolumn{2}{c}{\textbf{8\%}} \\ \cline{3-10}
\textbf{} & \textbf{} & Expansions & Time[ms] & Expansions & Time[ms] & Expansions & Time[ms] & Expansions & Time[ms] \\ \hline
\multirow{5}{*}{\textbf{7x7}} & \textbf{A*} & 91 & 197 & 139 & 366 & 273 & 710 & 729 & 2,136 \\
& \textbf{XMM} & 82 & 277 & 164 & 548 & 282 & 1,031 & 781 & 2,787 \\
& \textbf{XDFBnB} & 122 & 282 & 210 & 520 & 380 & 1,013 & 1,089 & 3,070 \\
& \textbf{BiX-DFBnB F2E} & 228 & 1,391 & 590 & 3,837 & 1,350 & 10,140 & 8,873 & 65,588 \\
& \textbf{BiX-DFBnB F2F} & \textbf{50} & \textbf{187} & \textbf{65} & \textbf{282} & \textbf{96} & \textbf{480} & \textbf{239} & \textbf{1,269} \\ \hline
\multirow{5}{*}{\textbf{7x8}} & \textbf{A*} & 148 & \textbf{383} & 390 & 1,102 & 652 & 2,124 & 1,475 & 4,855 \\
& \textbf{XMM} & \textbf{118} & \textbf{418} & 422 & 1,479 & 636 & 2,558 & 1,470 & 6,100 \\
& \textbf{XDFBnB} & 246 & 674 & 533 & 1,660 & 1,168 & 3,420 & 2,690 & 8,738 \\
& \textbf{BiX-DFBnB F2E} & 386 & 2,867 & 3,048 & 23,413 & 7,430 & 59,729 & 29,174 & 248,566 \\
& \textbf{BiX-DFBnB F2F} & 344 & 668 & \textbf{221} & \textbf{1,060} & \textbf{299} & \textbf{1,528} & \textbf{662} & \textbf{3,707} \\ \hline
\multirow{5}{*}{\textbf{7x9}} & \textbf{A*} & 224 & \textbf{695} & 396 & \textbf{1,195} & 1,414 & 5,635 & 1,536 & \textbf{5,781} \\
& \textbf{XMM} & 235 & 1,002 & 384 & 1,597 & 1,264 & 6,265 & 1,354 & 6,675 \\
& \textbf{XDFBnB} & 404 & 1,194 & 1,085 & 3,051 & 1,818 & 7,090 & 2,517 & 9,819 \\
& \textbf{BiX-DFBnB F2E} & 1,638 & 13,989 & 4,658 & 38,418 & 32,868 & 332,449 & 43,188 & 434,607 \\
& \textbf{BiX-DFBnB F2F} & \textbf{172} & 774 & \textbf{384} & \textbf{1,502} & \textbf{983} & \textbf{5,420} & \textbf{1,036} & \textbf{5,989} \\ \hline
\multirow{5}{*}{\textbf{8x8}} & \textbf{A*} & 204 & \textbf{546} & 490 & 1,612 & 1,619 & 5,458 & 2,704 & 10,971 \\
& \textbf{XMM} & 218 & 828 & 495 & 2,143 & 1,477 & 6,565 & 2,558 & 12,418 \\
& \textbf{XDFBnB} & 442 & 1,329 & 666 & 2,380 & 2,950 & 10,040 & 4,640 & 17,980 \\
& \textbf{BiX-DFBnB F2E} & 959 & 7,515 & 4,251 & 37,252 & 32,001 & 271,845 & 95,620 & 878,808 \\
& \textbf{BiX-DFBnB F2F} & \textbf{114} & 605 & \textbf{293} & \textbf{1,538} & \textbf{653} & \textbf{4,085} & \textbf{1,717} & \textbf{9,948} \\ \hline
\multirow{5}{*}{\textbf{8x9}} & \textbf{A*} & \textbf{372} & \textbf{1,145} & 1,253 & \textbf{4,455} & 2,269 & \textbf{8,830} & 8,307 & \textbf{35,688} \\
& \textbf{XMM} & 387 & 1,718 & 1,337 & 6,288 & 2,414 & 11,974 & 8,587 & 47,070 \\
& \textbf{XDFBnB} & 720 & 2,571 & 1,937 & 7,736 & 7,117 & 27,761 & 19,675 & 89,784 \\
& \textbf{BiX-DFBnB F2E} & 2,484 & 23,323 & 21,664 & 206,814 & 94,239 & 922,858 & 599,936 & 6,219,780 \\
& \textbf{BiX-DFBnB F2F} & 397 & 2,172 & \textbf{957} & 5,837 & \textbf{1,817} & 10,962 & \textbf{6,635} & 46,277 \\ \hline
\multirow{5}{*}{\textbf{9x9}} & \textbf{A*} & 1,053 & \textbf{3,646} & 2,707 & \textbf{11,225} & 5,343 & \textbf{26,503} & 15,100 & \textbf{81,790} \\
& \textbf{XMM} & 1,108 & 5,474 & 2,563 & 14,025 & 5,062 & 29,994 & 14,422 & 90,991 \\
& \textbf{XDFBnB} & 1,731 & 7,155 & 7,860 & 31,482 & 10,704 & 52,565 & 38,607 & 196,782 \\
& \textbf{BiX-DFBnB F2E} & 14,556 & 170,965 & 89,826 & 1,029,920 & 429,787 & 5,187,210 & 2,291,846 & 27,265,805 \\
& \textbf{BiX-DFBnB F2F} & \textbf{604} & 4,318 & \textbf{2,299} & 14,934 & \textbf{4,166} & 30,008 & \textbf{10,840} & 87,931 \\ \hline
\end{tabular}
}
\caption{Snake in Grids results.}\label{tab:SnakeGrids}
\end{table*}


\subsection{Coil-in-the-Box (CIB)}
The following table provides the raw numeric runtimes and expansions for finding and proving the optimal constrained paths in the $n$-dimensional hypercubes (4D through 7D).
\begin{table}[h]
\centering
\begin{tabular}{c|l|rr|rr}
\hline
\textbf{Dimension} & \textbf{Algorithm} & \multicolumn{2}{c|}{\textbf{Finding}} & \multicolumn{2}{c}{\textbf{Proving}} \\
\cline{3-6}
\textbf{} & \textbf{} & Expansions & Time[ms] & Expansions & Time[ms] \\
\hline
\multirow{5}{*}{\textbf{4D}} & \textbf{A*} & 18 & 14 & 22 & 14 \\
 & \textbf{XMM} & 9 & 9 & 9 & 9 \\
 & \textbf{XDFBnB} & 3 & 2 & 4 & 2 \\
 & \textbf{BiX-DFBnB F2E} & 2 & 3 & 2 & 3 \\
 & \textbf{BiX-DFBnB F2F} & \textbf{2} & \textbf{1} & \textbf{2} & \textbf{1} \\
\hline
\multirow{5}{*}{\textbf{5D}} & \textbf{A*} & 144 & 137 & 145 & 138 \\
 & \textbf{XMM} & 88 & 159 & 88 & 161 \\
 & \textbf{XDFBnB} & 25 & 33 & 50 & 73 \\
 & \textbf{BiX-DFBnB F2E} & 5 & 26 & 77 & 233 \\
 & \textbf{BiX-DFBnB F2F} & \textbf{5} & \textbf{11} & \textbf{9} & \textbf{26} \\
\hline
\multirow{5}{*}{\textbf{6D}} & \textbf{A*} & 19,698 & 48,423 & 19,699 & 48,625 \\
 & \textbf{XMM} & 12,223 & 53,668 & 34,001 & 120,859 \\
 & \textbf{XDFBnB} & 577 & \textbf{936} & 7,018 & 22,085 \\
 & \textbf{BiX-DFBnB F2E} & 2,347 & 19,507 & 163,561 & 1,375,600 \\
 & \textbf{BiX-DFBnB F2F} & \textbf{512} & 3,228 & \textbf{2,190} & \textbf{18,729} \\
\hline
\multirow{5}{*}{\textbf{7D}} & \textbf{A*} & $>$ 7,000,000 & $>$ 58,000,000 & --- & --- \\	
 & \textbf{XMM} & $>$ 6,500,000 & $>$ 69,000,000 & --- & --- \\ 	
 & \textbf{XDFBnB} & 1,310,050 & 8,834,600 & --- & --- \\
 & \textbf{BiX-DFBnB F2E} & $>$ 5,000,000 & $>$ 173,000,000 & --- & --- \\
 & \textbf{BiX-DFBnB F2F} & \textbf{3,265} & \textbf{56,960} & --- & --- \\
\hline
\end{tabular}

\caption{Coil-in-the-box results.}\label{tab:coil_in_box}
\end{table}

\clearpage
\begin{figure*}
    \centering
    \includegraphics[width=0.3\linewidth]{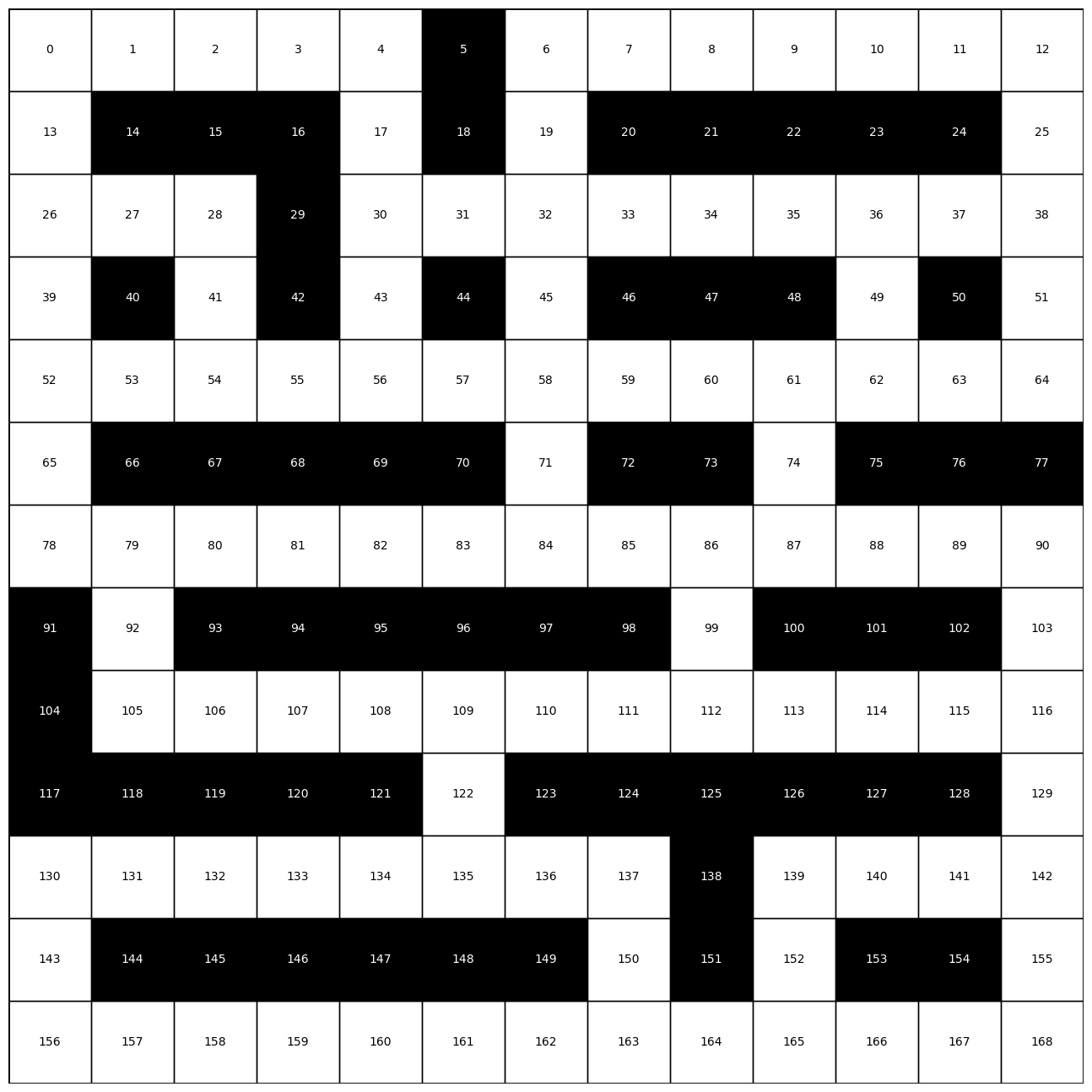}
    \caption{Original maze, with 0 open diamonds}
    \label{fig:placeholder}
\end{figure*}

\begin{figure*}
    \centering
    \includegraphics[width=0.3\linewidth]{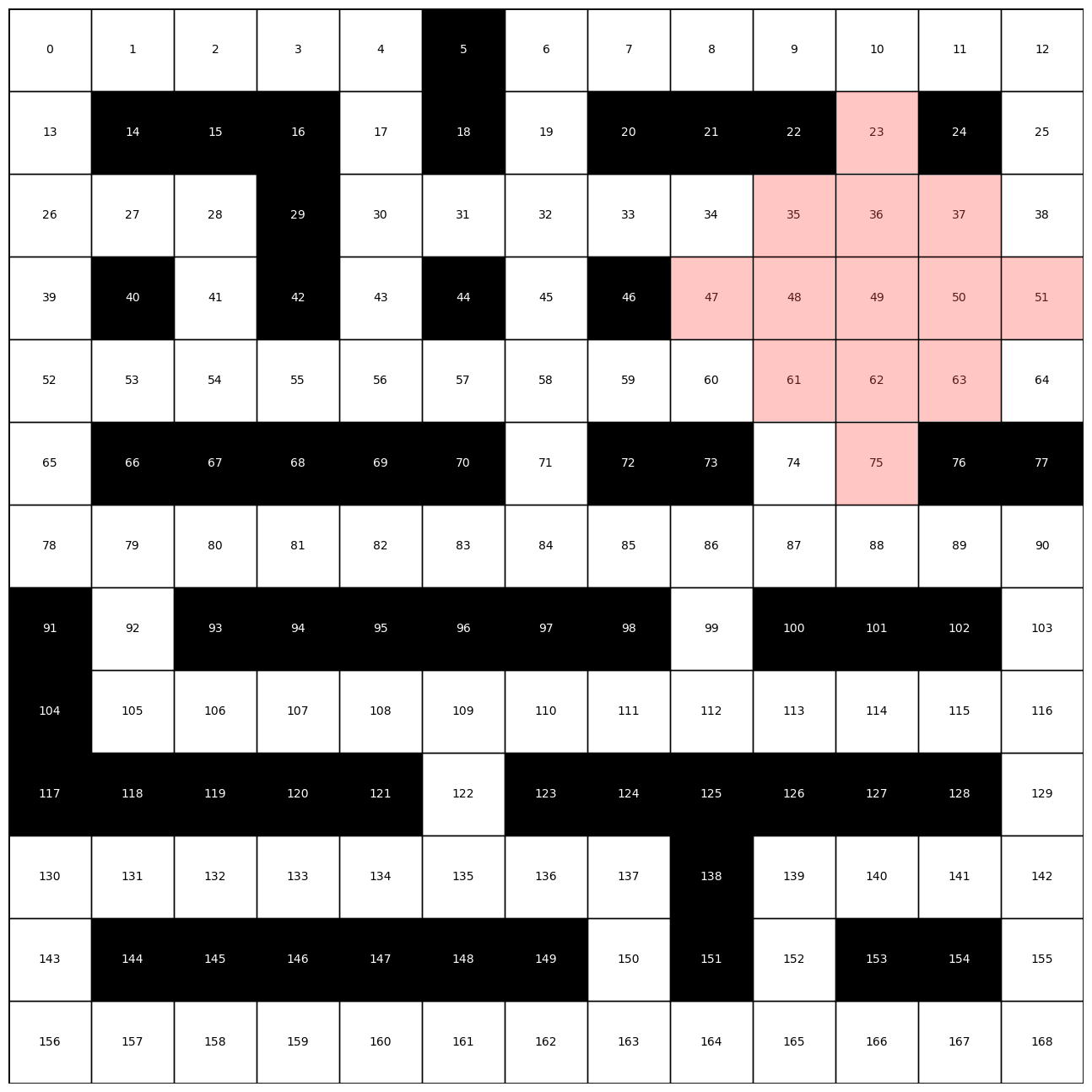}
    \caption{Maze with 1 open diamond}
    \label{fig:placeholder}
\end{figure*}
\begin{figure*}
    \centering
    \includegraphics[width=0.3\linewidth]{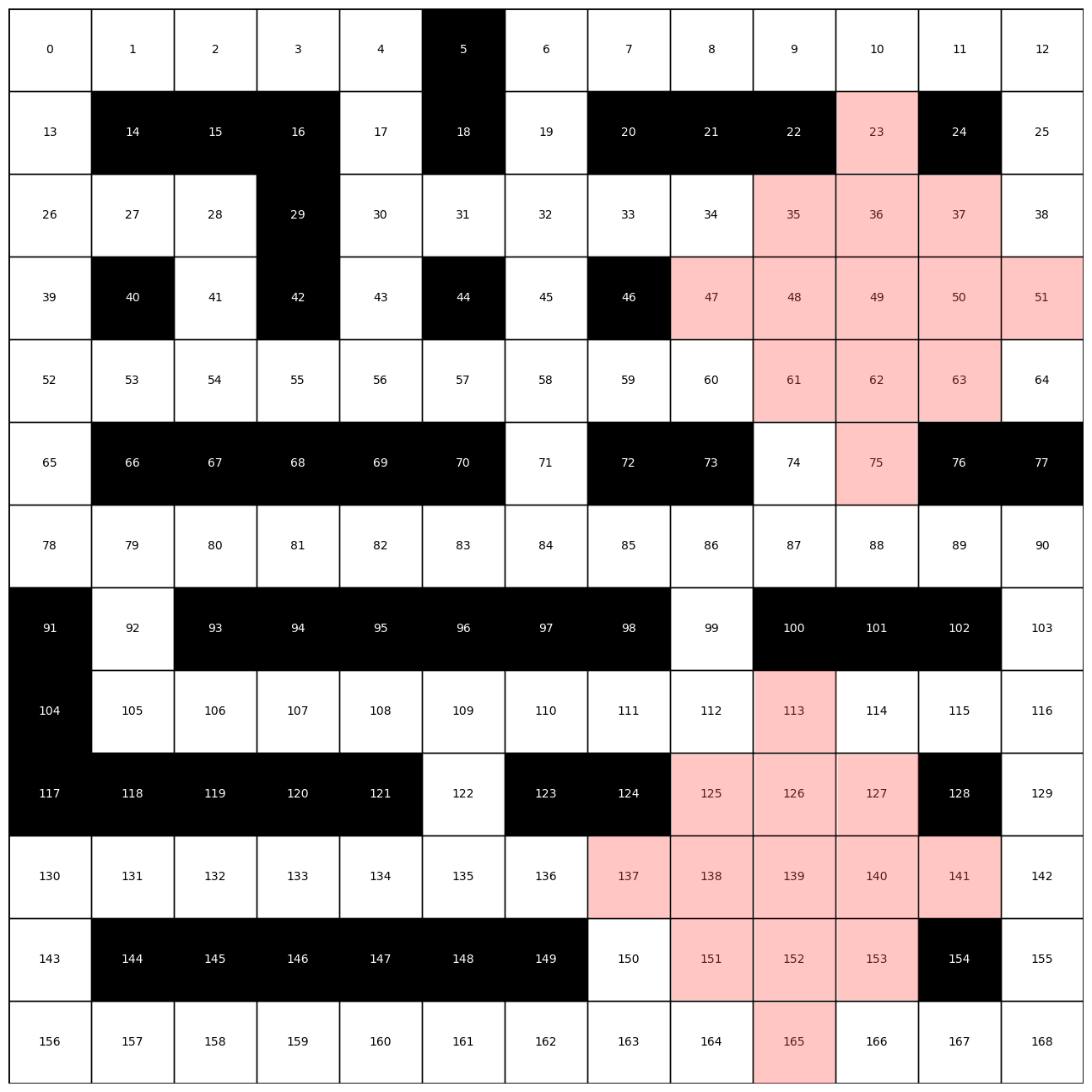}
    \caption{Maze with 2 open diamond}
    \label{fig:placeholder}
\end{figure*}

\end{document}